\documentclass[sigconf]{acmart}

\usepackage{algorithm}
\usepackage{algorithmic}

\usepackage{algorithm}
\usepackage{amsmath}
\usepackage{amsthm}
\usepackage{verbatim}
\usepackage{bm}
\usepackage{dsfont}
\usepackage{url}
\usepackage{array}
\usepackage[utf8]{inputenc}
\usepackage{mathtools}
\usepackage{enumitem}
\usepackage{balance}

\newcommand{\mt}{\mathsf{T}}
\newtheorem{corollary}{Corollary}
\newtheorem{claim}{Claim}

\newtheorem{assumption}{Assumption}

\DeclareMathOperator*{\argmax}{arg\,max}

\newtheorem{theorem}{Theorem}
\newtheorem{lemma}{Lemma}

\newtheorem{definition}{Definition}
\DeclarePairedDelimiter{\ceil}{\lceil}{\rceil}

\def \btheta {\bm \theta}
\def \bA {\mathbf{A}}
\def \bx {\mathbf{x}}
\def \bI {\mathbf{I}}
\def \bb {\mathbf{b}}
\def \cA {\mathcal{A}}
\def \cM {\mathcal{M}}

\newcommand{\compilehidecomments}{false}
\ifthenelse{ \equal{\compilehidecomments}{true} }{%
	\newcommand{\huazheng}[1]{}
}{
\newcommand{\huazheng}[1]{{\color{blue!50!black}  [\text{Huazheng:} #1]}}
}

\AtBeginDocument{%
  }

\copyrightyear{2022}
\acmYear{2022}
\setcopyright{acmlicensed}\acmConference[RecSys '22]{Sixteenth ACM Conference on Recommender Systems}{September 18--23, 2022}{Seattle, WA, USA}
\acmBooktitle{Sixteenth ACM Conference on Recommender Systems (RecSys '22), September 18--23, 2022, Seattle, WA, USA}
\acmPrice{15.00}
\acmDOI{10.1145/3523227.3546781}
\acmISBN{978-1-4503-9278-5/22/09}

\begin{document}

\title{Dynamic Global Sensitivity for Differentially Private Contextual Bandits}


\author{Huazheng Wang}
\authornote{Work was done while the first author was a PhD student at the University of Virginia.}
\email{huazheng.wang@oregonstate.edu}
\affiliation{%
  \institution{Oregon State University}
  \city{Corvallis}
  \state{Oregon}
  \country{USA}
}
\author{David Zhao}
\email{ dz6hu@virginia.edu}
\affiliation{
  \institution{University of Virginia}
  \city{Charlottesville}
  \state{Virginia}
  \country{USA}
}
\author{Hongning Wang}
\email{hw5x@virginia.edu}
\affiliation{
  \institution{University of Virginia}
  \city{Charlottesville}
  \state{Virginia}
  \country{USA}
 }


\begin{abstract}
We propose a differentially private linear contextual bandit algorithm,  via a tree-based mechanism to add Laplace or Gaussian noise to model parameters.
Our key insight is that as the model converges during online update, the global sensitivity of its parameters shrinks over time (thus named dynamic global sensitivity). Compared with existing solutions, our dynamic global sensitivity analysis allows us to inject less noise to obtain $(\epsilon, \delta)$-differential privacy with added regret caused by noise injection 
in $\tilde O(\log{T}\sqrt{T}/\epsilon)$. 
We provide a rigorous theoretical analysis over the amount of noise added via dynamic global sensitivity and the corresponding upper regret bound of our proposed algorithm.
Experimental results on both synthetic and real-world datasets confirmed the algorithm's advantage against existing solutions.
\end{abstract}


\begin{abstract}
Bandit algorithms have become a reference solution for interactive recommendation. However, as such algorithms directly interact with users for improved recommendations, serious privacy concerns have been raised regarding its practical use. In this work, we propose a differentially private linear contextual bandit algorithm,  via a tree-based mechanism to add Laplace or Gaussian noise to model parameters.
Our key insight is that as the model converges during online update, the global sensitivity of its parameters shrinks over time (thus named dynamic global sensitivity). Compared with existing solutions, our dynamic global sensitivity analysis allows us to inject less noise to obtain $(\epsilon, \delta)$-differential privacy with added regret caused by noise injection 
in $\tilde O(\log{T}\sqrt{T}/\epsilon)$. 
We provide a rigorous theoretical analysis over the amount of noise added via dynamic global sensitivity and the corresponding upper regret bound of our proposed algorithm.
Experimental results on both synthetic and real-world datasets confirmed the algorithm's advantage against existing solutions.

\end{abstract}

\maketitle

\section{Introduction}
Multi-armed bandit algorithms have become a reference solution for sequential decision-making; and they have been successfully applied to a wide variety of real-world applications such as recommendation \citep{LinUCB}, display advertisement \citep{li2010exploitation}, and clinical trials \citep{durand2018contextual}. In each iteration of such problems, a learner selects among a set of recommendation candidates, often referred as arms, and receives the corresponding reward after each selection. The learner's goal is to maximize the cumulative reward over time or equivalently to minimize regret. This requires the learner to both \emph{exploit} the currently estimated best arm and \emph{explore} among the arms to improve its estimation. Contextual bandits extend this setting to that the learner is given a context vector that encodes side information for reward estimation in each iteration. 

As bandit algorithms oftentimes directly work with user feedback, e.g., treating user clicks as reward, serious privacy concerns have been raised \citep{thakurta2013nearly, agarwal2017price, tossou2017achieving, shariff2018differentially}. We use recommender systems equipped with bandit algorithms as an example to illustrate its risk of privacy breach. In such a system, the algorithm chooses an item for a user, and the user decides whether to click the recommendation based on his/her true preference. Such click feedback is used to update the bandit model for improving its subsequent recommendations. As a result, any change in a user's preference promptly leads to changes in the algorithm's output, e.g., different sequences of recommended items. 
User's private information, e.g., his/her preference over items in this example, is thus revealed even if the click feedback is kept private. The goal of privacy protection is to prevent the algorithm's output from revealing a user's private information, such as item preferences. Real-world privacy breaches have been reported in Amazon's recommendation system \citep{calandrino2011you} and Facebook's advertisement system \citep{korolova2010privacy}, where an adversary can learn considerable side information about a user solely based on the system's recommendation sequence.  

Differential privacy \citep{dwork2006calibrating,dwork2014algorithmic} provides a rigorous guarantee that an algorithm's output has little dependency on the change of input at a single data point, which limits the amount of sensitive information an adversary can infer from the algorithm's output. It has been widely adopted in both industry and academia. The basic idea is to add a controlled level of noise to an algorithm's output, such that the subsequent change in its output caused by input change is indistinguishable from that caused by the added noise. The key is to determine the scale of the noise, which depends on the \emph{sensitivity} of the output given the change of input. However, the utility of a private algorithm decreases due to the injected noise. Under the same privacy budget, a good private algorithm adds less noise to its output, preserving more utility while achieving the same level of privacy.

Differential privacy has been studied in stochastic multi-armed bandits \citep{mishra2015nearly, tossou2016algorithms}, adversarial bandits \citep{tossou2017achieving}, and linear contextual bandits \citep{shariff2018differentially, neel2018mitigating}. 
Because differential privacy is immune to post-processing \citep{dwork2014algorithmic}, the output of a bandit algorithm (i.e., the sequence of selected arms) is differentially private once a private mechanism protects sensitive statistics at any internal stage of the algorithm, e.g., collected rewards or estimated bandit parameters. All existing differentially private  bandit algorithms add noise to reward or context vectors \citep{shariff2018differentially, neel2018mitigating}; but they fail to realize the role of bandit model's convergence in achieving privacy. Because the model parameter is converging during online update, the same reward change from the user (i.e., the input to the algorithm) will impose smaller impact to the model parameter in the later stage, and thus leads to less change in the algorithm's output. This property suggests that the model is \emph{less sensitive} to the change of input over time and thus less noise is needed for privacy protection in the later stage. 

Based on this important observation, we first study the dynamic nature of sensitivity of the bandit model parameters over time, which we refer to as \emph{dynamic global sensitivity} analysis. We apply the tree-based mechanism \citep{dwork2010differential, chan2011private} with dynamic global sensitivity to add Laplace noise to the model parameter of a linear UCB algorithm \citep{LinUCB} 
 to achieve the same level of differential privacy, but with much less added noise compared with previous approaches.
We rigorously analyze the level of noise added by the tree-based mechanism using dynamic global sensitivity and prove the upper regret bound of our proposed algorithm. The main contributions of this paper can be summarized as follows:
\begin{itemize}
    \item We quantify the dynamic global sensitivity of a linear bandit model based on its convergence property. As the model converges, it becomes less sensitive to the reward change over time, and thus less noise is needed to obtain the same level of privacy.
    \item We propose a tree-based mechanism with dynamic global sensitivity analysis. We show that this mechanism adds $O(\frac{1}{\epsilon}\log T \sqrt{\log T}\log{\frac{1}{\zeta}}/T )$ noise to the bandit model parameter at time $T$, and with probability $1-\zeta$, it achieves $(\epsilon, \delta)$-differential privacy. As a result, the added regret caused by noise injection is 
    $\tilde O(\log{T}\sqrt{T}/\epsilon)$ \footnote{$\tilde O(\cdot)$ omits the logarithmic terms of $d, 1/\lambda, 1/\zeta$, 1/$\delta$}.
    \item We also empirically evaluate our algorithm on both synthetic and real-world datasets and validate the algorithm's advantage against existing solutions.
\end{itemize}

\section{Related Work}

Multi-armed bandit problem was first studied in \cite{thompson1933likelihood} and \cite{robbins1952some}. \citet{UCB1} proposed the Upper Confidence Bound (UCB) method to solve the stochastic multi-armed bandit problem. We refer to \citep{bubeck2012regret,lattimore2018bandit} for a comprehensive survey regarding stochastic multi-armed bandits and its variants. In this paper, we focus on the linear contextual bandit problem, where each arm is associated with a context feature vector that encodes side information. The reward is governed by a linear function of the context vector, characterized by the corresponding model parameters. UCB-style solution for linear contextual bandit has been popularly studied in \citep{Auer02, LinUCB, Improved_Algorithm}.

Differential privacy \citep{dwork2006calibrating} provides a formal notion to quantify the amount of information regarding to an algorithm's input that an adversary could obtain by observing the algorithm's output. A common technique is to add Laplace or Gaussian noise to the algorithm's output. The scale of noise depends on both privacy level $(\epsilon,\delta)$ and \emph{sensitivity}, which is the change of algorithm's output caused by the change of its input. Originally studied for static databases, differential privacy was first extended to online setting for stream data in \citep{dwork2010differential, chan2011private}. Differentially private online learning methods have been studied for online convex optimization \citep{thakurta2013nearly,agarwal2017price,abernethy2017online} and bandit problems \citep{mishra2015nearly, tossou2017achieving,shariff2018differentially, neel2018mitigating, basu2019differential, wang2020global}. The key component of these solutions is the \emph{tree-based mechanism}, which was first proposed in \citep{dwork2010differential, chan2011private} for privately releasing sum statistics in stream data. There are some recent works studied local differential privacy for bandits~\cite{ren2020multi, zheng2020locally, han2021generalized} and DP for federated/distributed bandits~\cite{dubey2020differentially, dubey2020private}. We follow the notion of global differential privacy, and \emph{differential privacy} refers to the global notion in the rest of the paper.

In the bandit setting, an adversary can observe the selected arms (\emph{not} the reward). The goal of a private bandit algorithm is to keep the sequence of reward private. In other words, a change in the reward sequence should not lead to any sensible change in the selected arm sequence. Regarding differentially private contextual bandit, \citet{mishra2015nearly} proposed private versions of contextual UCB and Thompson Sampling algorithms, but they did not provide any theoretical analysis of the resulting algorithms. \citet{neel2018mitigating} provided regret analysis for a private LinUCB algorithm and proved the added regret introduced by their privacy mechanism is $\tilde O(\log^{2.5}T\sqrt{T}/\epsilon)$ to achieve $\epsilon$-differential privacy.  \citet{shariff2018differentially} studied the setting of making both context and reward private, and adopted a different privacy notion of joint differential privacy \cite{kearns2014mechanism}. However, these solutions do not directly add noise to the bandit model itself but to the input of the model that has \emph{constant} sensitivity. They thus cannot leverage the convergence property of a bandit model. We study the dynamic nature of sensitivity of bandit model during online update, and propose a tree-based mechanism with dynamic sensitivity to introduce less noise to the model parameter over time. Our proposed method adds additional regret to non-private LinUCB in a scale of $\tilde O(\log{T}\sqrt{T}/\epsilon)$ to achieve $(\epsilon, \delta)$-differential privacy. Pichapati et al. \cite{pichapati2019adaclip} proposed an adaptive differentially private SGD algorithm that has similar motivation of our work, i.e., analyzing the dynamics of sensitivity to preserve more
utility. But it focuses on a very different perspective: it dynamically (adaptively) computes sensitivity on different dimensions giving the historical gradient, while we compute the dynamic sensitivity based on model convergence.

\section{Method}
We first provide a brief overview of contextual bandit algorithms and differential privacy. We then provide the dynamic global sensitivity analysis of bandit algorithms with a tree based mechanism. Based on it, we present our differentially private linear contextual bandit algorithm and prove its upper regret bound.

\subsection{Preliminaries}\label{sec:preliminaries}
\textbf{Contextual Bandits.}
In a contextual bandit problem, an algorithm sequentially selects an arm $a_t$ from a candidate pool $\cA_t = \{a_{1,t}, a_{2,t}, ..., a_{k,t}\}$, and receives the corresponding reward $r_t$ afterwards. Each arm $a$ is associated with a context feature vector $\bx_a$ and its reward is governed by a function of the context feature vector and an unknown bandit model parameter $\btheta^*$. 
We made the following assumption  on context generation similar to  \citep{gentile2014online}.

\begin{assumption}\label{assumption:context}
Context vectors $C_t = \{\bx_{a} | a\in \cA_t\}$ are assumed to be generated i.i.d. conditioned on the past actions $\{a_1, .. a_{t-1}\}$ and contexts $\{C_1, .. C_{t-1}\}$ from a random process $X$ such that  $\mathbb{E}[XX^\mt]$ is full rank, with minimal eigenvalue $\lambda_0$ \footnote{We discuss the impact of this assumption in next section in detail.}.
\end{assumption}

The objective of a bandit algorithm is to maximize its cumulative reward (or equivalently minimize the regret) over a finite time horizon $T$. To simplify our discussion, we assume the observed reward is a linear function of feature vector $\bx_a$ and bandit model parameter $\btheta^*$ with noise, i.e., $r_t = \bx^{\mt}_{a_t} \btheta^* + \gamma_t$, where $\gamma_t$ is a sub-Gaussian feedback noise. 
We evaluate a bandit algorithm by its pseudo-regret, which is defined as 
\begin{equation}
\label{eq_pregret}
\text{Regret}(T) = \sum_{t=1}^T(\bx^{\mt}_{a^*_t}\btheta^* - \bx^{\mt}_{a_t}\btheta^*),
\end{equation}
where $a_t^*$ is the best arm according to $\btheta^*$.
The LinUCB algorithm \citep{LinUCB,Improved_Algorithm} is a well-studied solution for linear bandit. It solves the bandit model parameter $\hat\btheta$ using ridge regression, i.e., $\hat\btheta = \bA_t^{-1} \bb_t$, where $\bA_t=\sum_{i=1}^{t-1} \bx_{a_i} \bx_{a_i}^\mt + \lambda\bI, \bb_t = \sum_{i=1}^{t-1} \bx_{a_i} r_i$, $\lambda$ is the L2-regularization coefficient; and Upper Confidence Bound is used to select an arm.

\textbf{Differential Privacy.}
For a contextual bandit problem, denote $S = \{r_{1:T}\}$ as the reward sequence. $S'$ is considered as an adjacent neighboring sequence of $S$, if it only differs from $S$ at one point of reward $r_i$. The output $\mathcal{O}$ of a bandit algorithm is the sequence of its selected arms, i.e., $\{a_{1:T}\}$.
\begin{definition}[Differential Privacy \citep{dwork2006calibrating}]\label{def:dp}
A randomized algorithm $\cM$ is $(\epsilon,\delta)$-differentially private if for any adjacent neighboring sequences $S, S'$ and output $\mathcal{O}$,
\begin{align*}
\mathbb{P}\left(\cM(S) \in \mathcal{O} \right) \leq e^{\epsilon}\mathbb{P}\left(\cM(S') \in \mathcal{O} \right) +\delta
\end{align*}
When $\delta=0$, we say algorithm $\cM$ is $\epsilon$-differentially private.
\end{definition}

Differential privacy ensures the adversary observes almost the same output of a private algorithm in a probabilistic sense, if one input data point is changed, where the similarity between outputs is evaluated by $\epsilon$ and $\delta$. Laplace and Gaussian noise are commonly used as additive noise to protect the output, where the noise scale is related to the privacy requirement $(\epsilon,\delta)$ and the \emph{global sensitivity} of algorithm's output caused by the change of input. We formally define global sensitivity below.

\begin{definition}[Global Sensitivity \citep{dwork2006calibrating}]\label{def:sensitivity}
For any adjacent neighboring sequences $S, S'$, global sensitivity of a function $f$ is defined as,
\begin{align*}
\Delta_f = \max_{S, S'}\lvert f(S) - f(S') \rvert
\end{align*}
\end{definition}
Since differential privacy is immune to post-processing \citep{dwork2014algorithmic}, previous solutions of private linear bandits  \citep{mishra2015nearly, neel2018mitigating} add noise to  $\sum_t \bx_{a_t} r_t$ to obtain privacy, which consequently makes the model parameter $\btheta_t$ and the selected arms $\{a_t\}^T_{t=1}$ private. It is obvious that global sensitivity of $b_t=\sum_t \bx_{a_t} r_t$ is a constant $L$ for all round $t$, if we assume $\lVert\bx\rVert_2 \leq L$ and $|r| \leq 1$, without loss of generality. In other words, constant amount of noise has to be added in each round to achieve differential privacy.

\subsection{ Dynamic Global Sensitivity Analysis for Tree-based Mechanism}\label{sec:sensivitity}
\begin{algorithm}[tb!]
   \caption{Tree-based mechanism with Dynamic Global Sensitivity}
   \label{alg:tree}
\begin{algorithmic}[1]
\STATE \textbf{Inputs:} $t, T, \epsilon, \delta, L, \lambda_0$
\STATE \textbf{Initialize:} $\epsilon' = \epsilon/\log_2 T$, $\delta' = \delta/\log_2  T$, $\eta_t = 0$
\STATE Let $b$ be the $\ceil{\log_2 T}+1$-bit binary representation of $t$ 

\FOR {i = 0 to $\ceil{\log_2 T}+1$}
\IF {$b_i=1$}
    \STATE $\eta_t = \eta_t + \text{Lap}(\frac{\Delta_{t-2^i+1}}{\epsilon'})$ 
\ENDIF
\ENDFOR
\STATE Return noise $\eta_t$
\end{algorithmic}
\end{algorithm}

Different from previous work, we directly add noise to the estimated bandit model parameter $\hat\btheta_t$ after each round of model update. As $\hat\btheta_t$ converges over time, the sensitivity of it decreases consequently, which we name as \emph{dynamic global sensitivity}. We first quantify such sensitivity of $\hat\btheta_t$ and then discuss how to combine it with the tree-based mechanism.

\begin{lemma}[Sensitivity of estimated bandit model parameter $\hat\btheta_t$] \label{lemma:sensitivity}
Let $\hat\btheta_t$ and $\hat\btheta'_t $ be parameter estimations of adjacent neighboring reward sequences $S$ and $S'$, assuming context vectors $\{\bx_{1,t},..\bx_{K,t}\}$ are generated according to Assumption~\ref{assumption:context}.
With probability at least $1-\delta$, the dynamic global sensitivity of bandit model parameter $\hat\btheta_t$ at time $t$ is bounded as, 
\begin{align}
\Delta_t = \max_{S, S'}\lVert \hat\btheta_{t+1} - \hat\btheta'_{t+1}\rVert_2 \leq \frac{2 L}{\lambda'}
\end{align}
where $\lambda' = \lambda_0 t/4 - 8 \log((t+3)/\delta) -2\sqrt{t \log((t+3)/\delta)}$ is the lower bound of minimum eigenvalue of matrix $\bA_t = \sum_t \bx_{a_t} \bx_{a_t}^\mt + \lambda\bI$. \footnote{Note that $\lambda$ is the coefficient for of $L_2$ regularization, and it not related with $\lambda_0$.} We use a simplified bound $\lambda' = \lambda_0 t/16$ when $t>32\log(1/\delta)/\lambda_0$, which gives us $\Delta_t \leq 32L/\lambda_0$.
\end{lemma}
\textit{Proof Sketch.} Since $\hat\btheta_t = \bA_t^{-1} \sum_t \bx_{a_t} r_t$, the key idea is to quantify the minimum eigenvalue of matrix $\bA_t$. Here we adopt the i.i.d. assumption of context vectors from Theorem 1 of \cite{gentile2014online} to analyze the corresponding eigen system and the key idea is to use a concentration inequality to bound the eigenvalue. 
Note that we need to avoid  negative sensitivity when $t$ is small. It can be derived from the formula of $\lambda'$ that by choosing probability $\delta > 2\log(3)/(\sqrt{1+2\lambda_0} + 1)$, the sensitivity is guaranteed to be positive for small $t$. For large $t$, we have shown in the lemma that the sensitivity is guaranteed to be positive when $t>32\log(1/\delta)/\lambda_0$.

Lemma \ref{lemma:sensitivity} provides the dynamic global sensitivity analysis of estimated bandit model parameter $\hat\btheta_t$. The most important property of it is that it is monotonically shrinking over time. This indicates less noise is needed in later stage for privacy protection. We should emphasize that our analysis is for \emph{global} sensitivity. A similar but quite different notion is \emph{local} sensitivity (and smoothed sensitivity) proposed by \cite{nissim2007smooth}, which studies the dynamic nature of sensitivity given different input $S$, i.e., $\max_{S'}\lvert f(S) - f(S') \rvert$. Our dynamic sensitivity analysis is not conditioned on any reward sequence, but based on the convergence property of bandit model parameter $\hat\btheta_t$. And thus it holds for any reward sequence as input.

\noindent{\textbf{Remark}} Our Assumption~\ref{assumption:context} on the context vectors used in Lemma \ref{lemma:sensitivity} follows \cite{gentile2014online} and is used to bound the rate of shrinkage of the bandit model's sensitivity. 
Similar assumptions are also discussed in \cite{kannan2018smoothed} from the perspective of perturbation on context vectors via a smoothed analysis. And in the appendix, we show that a similar sensitivity analysis can easily lead to a differentially private version of algorithm proposed in \cite{kannan2018smoothed}. In the meanwhile, we also note that this environment assumption generally holds in practice, especially in a system with time-varying arm sets, as extensively discussed and studied in \citep{kannan2018smoothed}. 
However, even if the context vectors are generated by an adversary instead of a random process, we can still achieve the same guarantee on the minimal eigenvalue $\lambda_0$ and the same result in Lemma \ref{lemma:sensitivity} by perturbing the context vector $\bx_a$ with a Gaussian noise sampled from $\mathcal{N}(0, \sigma^2)$, such that $\lambda_0 \geq \Omega(\sigma^2)$, based on Lemma 3.7 in \citep{kannan2018smoothed}. In the following, we assume Assumption~\ref{assumption:context} holds if not specified. And in Corollary \ref{corollary:eigen}, we discuss the impact on regret if the algorithm needs to perturb the context vector when the assumption does not hold.

The analyzed sensitivity provides us the level of noise needed for differential privacy. We present a tree-based mechanism with dynamic global sensitivity in Algorithm \ref{alg:tree}, which takes advantage of the sublinear property of exponential function to further reduce the amount of added noise. Basically, the idea of tree-based mechanism is to view the sum statistics as $\ceil{\log_2 T}+1$ partial sums.  Laplace noise is added to each partial sum to achieve $(\epsilon',\delta')$-differential privacy where $\epsilon' = \epsilon/\log_2 T$, $\delta' = \delta/\log_2  T$. Here we follow the notion of $(\epsilon',\delta')$-differential privacy because our sensitivity analysis in Lemma \ref{lemma:sensitivity} is a high probability bound and holds with probability $1-\delta'$.  Based on composition theorem of differential privacy \citep{dwork2010boosting,kairouz2017composition}, the final output, the sum of $(\epsilon',\delta')$-differentially private partial sums, is $(\epsilon,\delta)$-differentially private.  The partial sums are segmented by the tree representation. We refer to \cite{dwork2010differential,chan2011private} for detailed discussion of tree-based mechanism. 

However, previous differentially private bandit algorithms with tree-based mechanism only protect variables that have constant sensitivities, such as sum statistics of $\sum_t \bx_{a_t} r_t$ for contextual bandits  \citep{mishra2015nearly,neel2018mitigating, shariff2018differentially} or sum of rewards for non-contextual bandits  \citep{tossou2016algorithms, tossou2017achieving}, instead of our dynamic global sensitivity. In Algorithm \ref{alg:tree}, we scale the noise of partial sum $[t-2^{i+1}, t-2^i+1]$ with its dynamic sensitivity $\Delta_{t-2^i+1}$ as shown in line 6.

We first show the privacy guarantee of Algorithm \ref{alg:tree} and then study its utility, i.e., the total amount of noise added by Algorithm \ref{alg:tree}.  

\begin{theorem}[Privacy]\label{theorem:privacy} Algorithm \ref{alg:tree} with dynamic global sensitivity $\Delta_i$ defined in Lemma \ref{lemma:sensitivity} is
$(\epsilon,\delta)$-differentially private.
\end{theorem}

\begin{theorem}[Utility]\label{theorem:utility} For a finite time horizon $T$, at time $t$, Algorithm \ref{alg:tree} with dynamic global sensitivity of bandit model parameter $\hat\btheta$ adds $O(\frac{L}{\epsilon}\log T \sqrt{\log t}\log{\frac{1}{\zeta}}/t)$ noise with probability $1-\zeta$ to achieve $(\epsilon,\delta)$-differential privacy.
\end{theorem}

\textit{Proof Sketch.} 
We follow the definition of differential privacy and use a similar proof as Theorem 3 of \citep{thakurta2013nearly} to get the privacy guarantee. The utility analysis is based on Lemma \ref{lemma:sensitivity} and the property of sum of samples from Laplace distributions. 

All existing tree-based mechanism for private bandit solutions add noise at a scale of $O(\frac{L}{\epsilon}\log T \sqrt{\log t}\log{\frac{1}{\zeta}})$ based on constant sensitivity \citep{chan2011private} . By leveraging the dynamic global sensitivity analysis in Lemma \ref{lemma:sensitivity}, our solution adds less noise while achieving the same level of privacy.

Moreover, our finite-time analysis assumes the knowledge of time horizon $T$ beforehand; but it can also be directly extended to the infinite/unknown time horizon by replacing tree-based mechanism with the hybrid mechanism \citep{chan2011private}. Although in this paper our focus is the contextual bandit problem, our dynamic global sensitivity analysis can also be generalized to other settings such as online convex optimization, where we can leverage the algorithm's convergence property to add decreasing noise to the model parameter for better utility.

\subsection{Differentially Private LinUCB}\label{sec:plinucb}

\begin{algorithm}[tb]
   \caption{Private LinUCB with Dynamic Global Sensitivity}
   \label{alg:plinucb}
\begin{algorithmic}[1]
    \STATE \textbf{Inputs:} $\epsilon, \delta, \lambda, T, L, \lambda_0$
 \STATE \textbf{Initialize:} $\bA_1 \gets \lambda \bI$, $\bb_1 \gets \mathbf{0}$
\FOR{ $t=1$ to $T$}	
\STATE Observe context vectors $C_t = \{\bx_{a} | a\in\cA_t\}$		
\STATE Take action $a_{t}=\argmax_{a\in\cA_t}  \bx_a^\mt\hat\btheta^p_t+ \alpha_t \lVert \bx_a \rVert_{\bA^{-1}_t}$
\STATE Observe reward $r_{a_{t}}$
\STATE $\bA_{t+1} \gets \bA_t +  \bx_{a_t} \bx_{a_t}^\mt$
\STATE $\bb_{t+1} \gets \bb_t + \bx_{a_t}r_{a_{t}}$
\STATE Sample noise $\eta_t \sim \text{TreeMechanism}(t, \epsilon,\delta) $
\STATE $\hat{\btheta}^p_{t+1} \gets \bA_{t+1}^{-1} \bb_{t+1} + \eta_t$
\ENDFOR
\end{algorithmic}
\end{algorithm}

We now provide a differentially private linear contextual algorithm built on top of the tree-based mechanism with dynamic global sensitivity. The details of this algorithm is described in Algorithm \ref{alg:plinucb}. At round $t$, the algorithm receives a set of arms, with each arm associated with a context vector $\bx_a$. Different from previous solutions \citep{mishra2015nearly,neel2018mitigating} that add noise to $\bb_t=\sum_t \bx_{a_t} r_t$, our algorithm directly adds noise to the bandit parameter for privacy (i.e., line 10), and uses the private model parameter $\hat\btheta^p_t$ for arm selection (i.e., line 5). The selected sequence of arms are proved to be $(\epsilon,\delta)$-differentially private to the reward sequence.

\begin{claim} The sequence of selected arms $\{a_t:t\in[1..T]\}$ by Algorithm \ref{alg:plinucb} is $(\epsilon, \delta)$-differentially private. 
\end{claim}
\begin{proof}
In line 9-10 we use Algorithm \ref{alg:tree} to add noise and keep $\hat\btheta^p_{t+1}$ $(\epsilon, \delta)$-differentially private. Because differential privacy is post-processing invariant \citep{dwork2014algorithmic}, the sequence of selected arms $\{a_t:t\in[1..T]\}$ produced by $\hat\btheta^p_{t}$ is thus also $(\epsilon, \delta)$-differentially private.
\end{proof}

We now specify the confidence bound of reward estimation with $\hat\btheta^p_{t}$, which we use for arm selection in line 5 of Algorithm \ref{alg:plinucb}.

\begin{lemma}[Confidence Bound] \label{lemma:cb}
Following Assumption~\ref{assumption:context} and assuming $\lVert\btheta^*\rVert_2\leq S$ , with probability at least $1-2\zeta$, confidence bound of the estimated reward $\bx^\mt \hat\btheta_t^p$ is
\begin{align}
\label{eq:cb}
\text{CB}_t(\bx) = \lVert\bx^\mt \hat\btheta_t^p - \bx^\mt \btheta^*\rVert \leq \alpha_t \lVert \bx \rVert_{\bA^{-1}_t} 
\end{align} 
where we define $\alpha_t$ as the upper bound of $\lVert\hat\btheta_t^p - \btheta^*\rVert_{\bA_t} $. We have
\begin{align}
\label{eq:alpha}
\lVert\hat\btheta_t^p - \btheta^*\rVert_{\bA_t} 
\leq  \frac{L}{\epsilon}\log T \sqrt{\log t}\log({\frac{1}{\zeta}})/\sqrt{t} 
+ \sqrt{d\log{\frac{1+tL^2\lambda}{\zeta}}}+\sqrt{\lambda}S 
\end{align} 
\end{lemma}
\begin{proof}
Eq. \eqref{eq:cb} is obtained by the Cauchy–Schwarz inequality. To bound $\alpha_t$, we apply the triangle inequality: $\lVert\hat\btheta_t^p - \btheta^*\rVert_{\bA_t}\leq \lVert\hat\btheta_t^p - \hat\btheta_t\rVert_{\bA_t} + \lVert\hat\btheta_t - \btheta^*\rVert_{\bA_t}$, where $\hat\btheta_t$ is the non-private estimate of $\btheta^*$. The first term is bounded according to Theorem \ref{theorem:utility} under $\bA_t$ norm with probability at least  $1-\zeta$. The second term is the confidence ellipsoid of non-private LinUCB and it can be bounded by Theorem 2 of  \cite{Improved_Algorithm} with probability at least $1-\zeta$. By taking a union bound, the inequality holds with probability at least $1-2\zeta$
\end{proof}

Lemma \ref{lemma:cb} gives a tight construction of uncertainty regarding reward estimation $\bx^\mt \hat\btheta_t^p$ and model parameter estimation $\lVert\hat\btheta_t^p - \btheta^*\rVert_{\bA_t}$. Note that comparing to the non-private LinUCB, the confidence bound of model estimation $\alpha_t$ in our algorithm is relaxed by an additional term $\frac{L}{\epsilon}\log T \sqrt{\log t}\log({\frac{1}{\zeta}})/\sqrt{t}$, which captures the upper bound of uncertainty caused by the privacy-preserving noise $\eta_t$ introduced by Algorithm \ref{alg:tree}.

Now we provide a gap-independent regret bound of Algorithm \ref{alg:plinucb} in following theorem. 
\begin{theorem}[Regret Bound]\label{theorem:regret}
Following Assumption~\ref{assumption:context}, the pseudo-regret of Algorithm \ref{alg:plinucb} up to time $T$ can be bounded by,
\small
\begin{align}
Regret(T) \leq & 2\left(\sqrt{d\log{\frac{1+TL^2\lambda}{\zeta}}}+\sqrt{\lambda}S\right)\sqrt{dT\log{(\lambda+\frac{LT}{d})}} \nonumber\\
&+ \left(2\frac{3^{1.5}L}{e^{1.5}\lambda_0\epsilon}\log({\frac{1}{\zeta}}) \sqrt{dT\log{(\lambda+\frac{LT}{d})}} + 32\log(1/\delta)/\lambda_0 \right)\end{align}
\normalsize
with probability at least $1-2\zeta$. 

\end{theorem}

\textit{Proof Sketch.} We rewrite the cumulative regret of LinUCB defined in Eq. \eqref{eq_pregret} by $\sum_{t=1}^T 2\text{CB}_t(\bx_{a_t})$, using the definition of confidence bound and UCB arm selection strategy in line 5 of Algorithm \ref{alg:plinucb}. We apply the Cauchy-Schwarz inequality to bound $\sum_t \alpha_t$ and $\sum_t \lVert \bx_{a_t} \rVert_{\bA^{-1}_t} $. According to Eq. \eqref{eq:alpha} in Lemma \ref{lemma:cb}, we can separate the bound of $\alpha_t$ into two terms: one is the confidence bound of the original LinUCB and the other is the bound of injected noise. We use Lemma 11 of  \cite{Improved_Algorithm} to bound $\sum_t \lVert \bx_{a_t} \rVert_{\bA^{-1}_t} $. 

The first term of regret in Theorem \ref{theorem:regret} is the same as the regret of non-private LinUCB, which is caused by its parameter estimation uncertainty and exploration in arm selection. The second term is the added regret introduced for privacy based on dynamic global sensitivity. The private LinUCB algorithm in \citep{neel2018mitigating} adds noise to $\bb_t$ and incurs additional regret in the order of $\tilde O(\log^{2.5}{T}\sqrt{T}/\epsilon)$ to achieve $\epsilon$-differential privacy, 
while our algorithm introduces additional regret in the order of $\tilde O(\log{T}\sqrt{T}/\epsilon)$ to achieve $(\epsilon, \delta)$-differential privacy ignoring logarithmic terms. The dependency of additional regret on $\delta$ is $O(\log(1/\delta))$. We note that our method leverages the  convergence of parameter estimation to inject less noise; and since it is a high probability analysis, the relaxed $(\epsilon, \delta)$-differential privacy notion is required. Later, we also validate our theoretical analysis of the regret reduction in our empirical evaluations. We currently cannot prove if our regret is optimal with respect to T, since the lower bound of the differentially private linear bandit problem is still an open problem. We note the lower bound discussed in \cite{shariff2018differentially} is not applicable, because their privacy definition includes context and it is different from ours that focuses on the privacy of reward. We consider deriving the lower bound of this problem as an important future direction.

In general, the regret bound of contextual bandits could be categorized into gap-independent bound and gap-dependent bound. Our provided analysis is for a gap-independent bound and it is dominated by the term $O(\sqrt{T})$ after ignoring the logarithmic terms. The gap-dependent regret bound of a bandit algorithm is usually in the order of $O(\log T /\mu_{1,k})$, where $\mu_{1,k}$ is the minimal reward difference between the best arm and any sub-optimal arm $k$. We leave the analysis of gap-dependent bound of our proposed algorithm as future work. 

Note that Algorithm~\ref{alg:plinucb} takes $L$ and $\lambda_0$ as input. From a theoretical perspective, the knowledge of the maximum L2-norm of context features $L$ is a commonly made assumption in linear bandits, which is required to derive important model parameters such as the size of confidence ellipsoid $\alpha_t$ in LinUCB~\cite{Improved_Algorithm}. From a practical view, one can normalize the context feature vectors or observe them ahead of the time to empirically estimate $L$ and $\lambda_0$.

When the context vectors are not generated from the required random process with minimal eigenvalue $\lambda_0$, we can add an additional Gaussian perturbation on the context to achieve this condition similar to the idea in \citet{kannan2018smoothed}. This leads to the following auxiliary result.
\begin{corollary}\label{corollary:eigen}
Following the assumptions in Lemma 1 and 2 about the context vectors, by perturbing every context vector $\bx_a$ with $z\sim \mathcal{N}(0, \sigma^2\bI) $, Algorithm \ref{alg:plinucb} is $(\epsilon, \delta)$-differentially private. The pseudo-regret of Algorithm \ref{alg:plinucb} up to time $T$ can be bounded by
\small
\begin{align}
Regret(T) \leq& 2\frac{\sqrt{d\log{\frac{1+TL^2\lambda}{\zeta}}}+\sqrt{\lambda}S}{\sigma^2}\sqrt{dT\log{(\lambda+\frac{LT}{d})}} \nonumber\\
&+ \left(2\frac{3^{1.5}L}{e^{1.5}\sigma^2\epsilon}\log({\frac{1}{\zeta}}) \sqrt{dT\log{(\lambda+\frac{LT}{d})}} + 32\log(1/\delta)/\sigma^2 \right)
\end{align}
\normalsize
with probability at least $1-2\zeta$.

\end{corollary}

Here we notice that if we do not have the assumption on the generation of context vectors, while instead perturbing the context vectors to guarantee Lemma \ref{lemma:sensitivity}, the regret would be larger than the one in Theorem \ref{theorem:regret}, especially when variance $\sigma^2$ is small. However note that the order of the regret in $T$ is still the same. The proof of this corollary can be found in the appendix.

\begin{figure*}[t]
\centering
\begin{tabular}{ >{\centering\arraybackslash}m{4.5cm} >{\centering\arraybackslash}m{4.5cm}
>{\centering\arraybackslash}m{4.5cm}}
\hspace{-0.4cm}
\includegraphics[width=4.8cm]{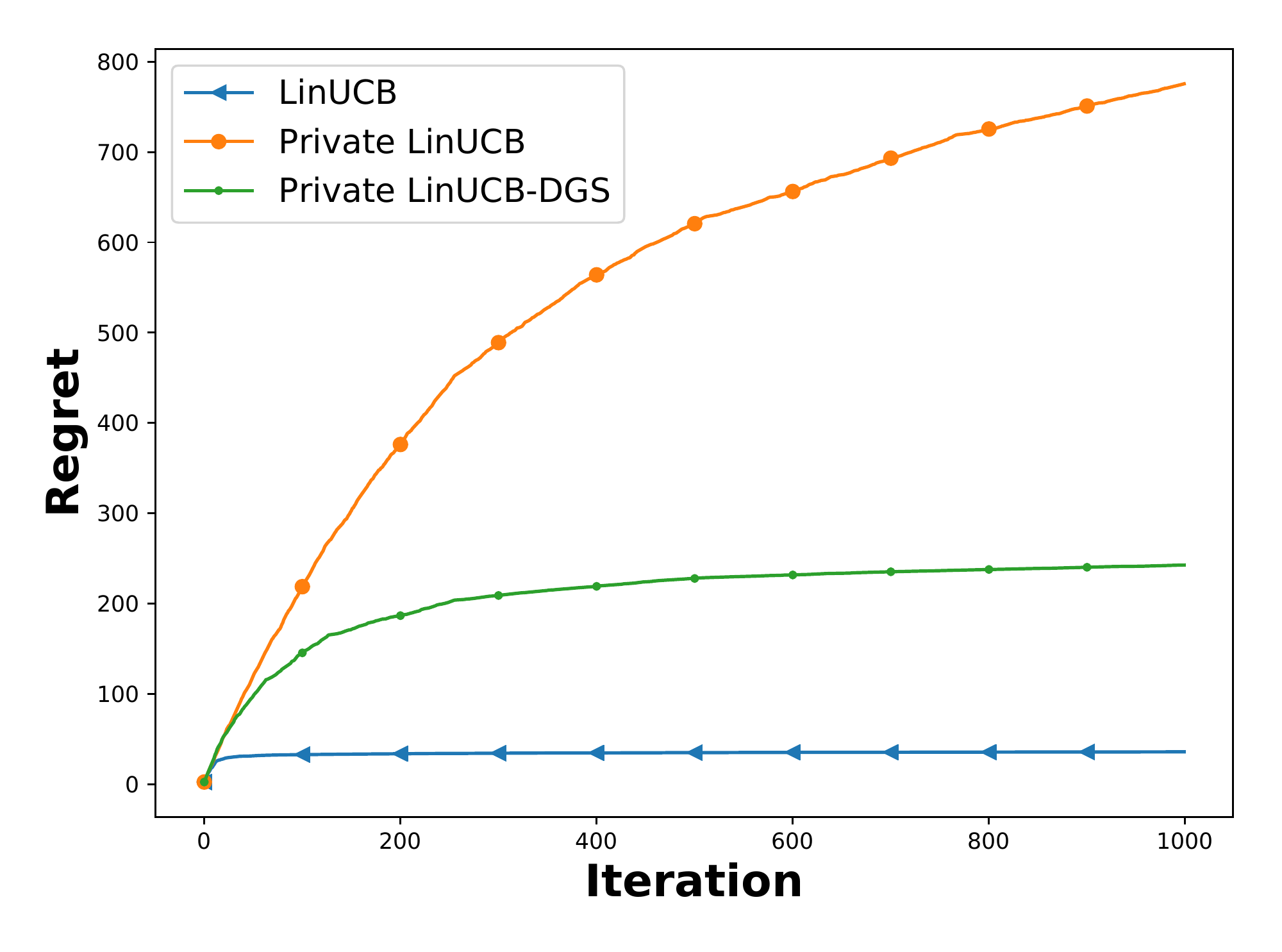} & 
\includegraphics[width=4.8cm]{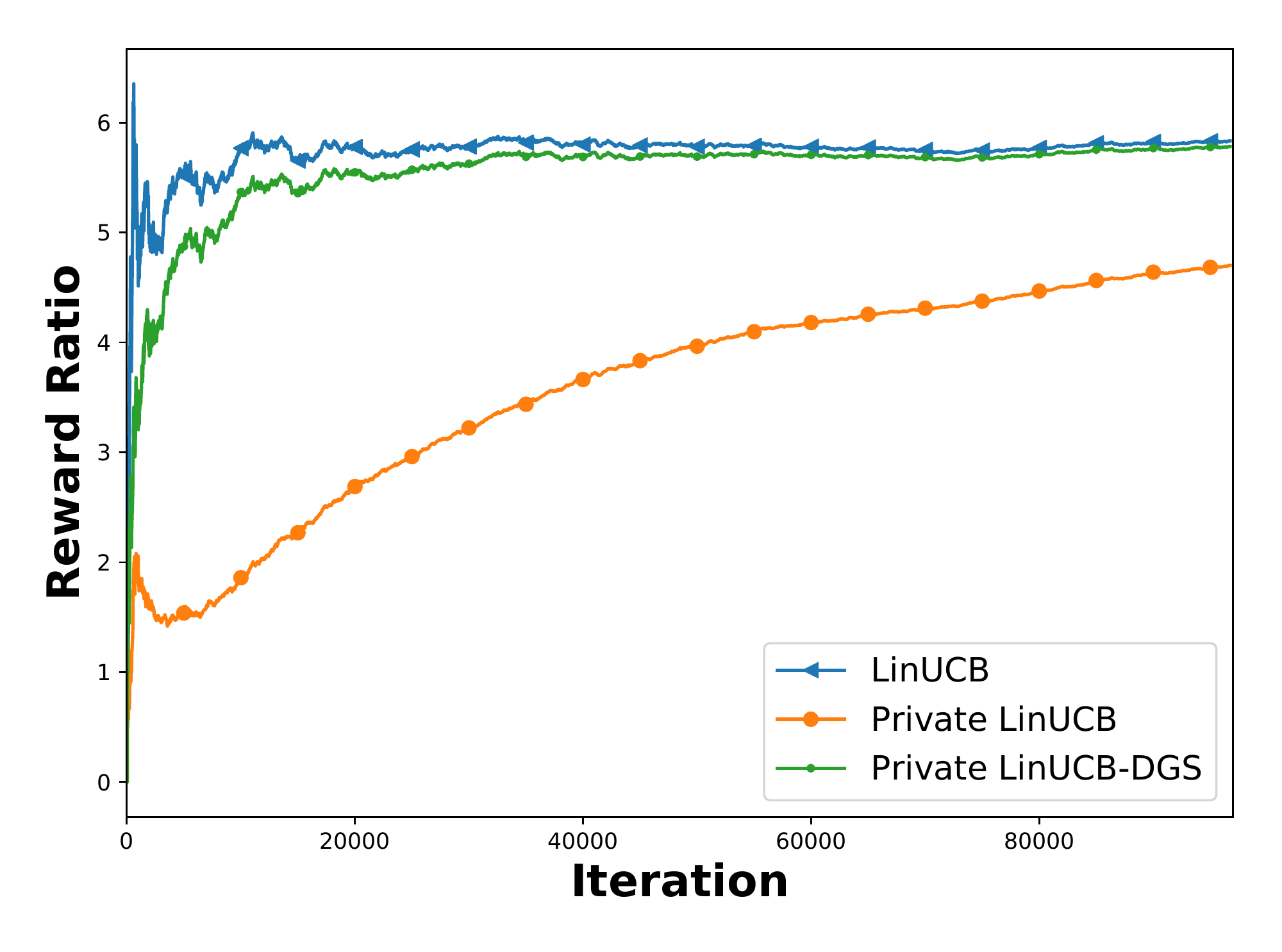}  &
\includegraphics[width=4.8cm]{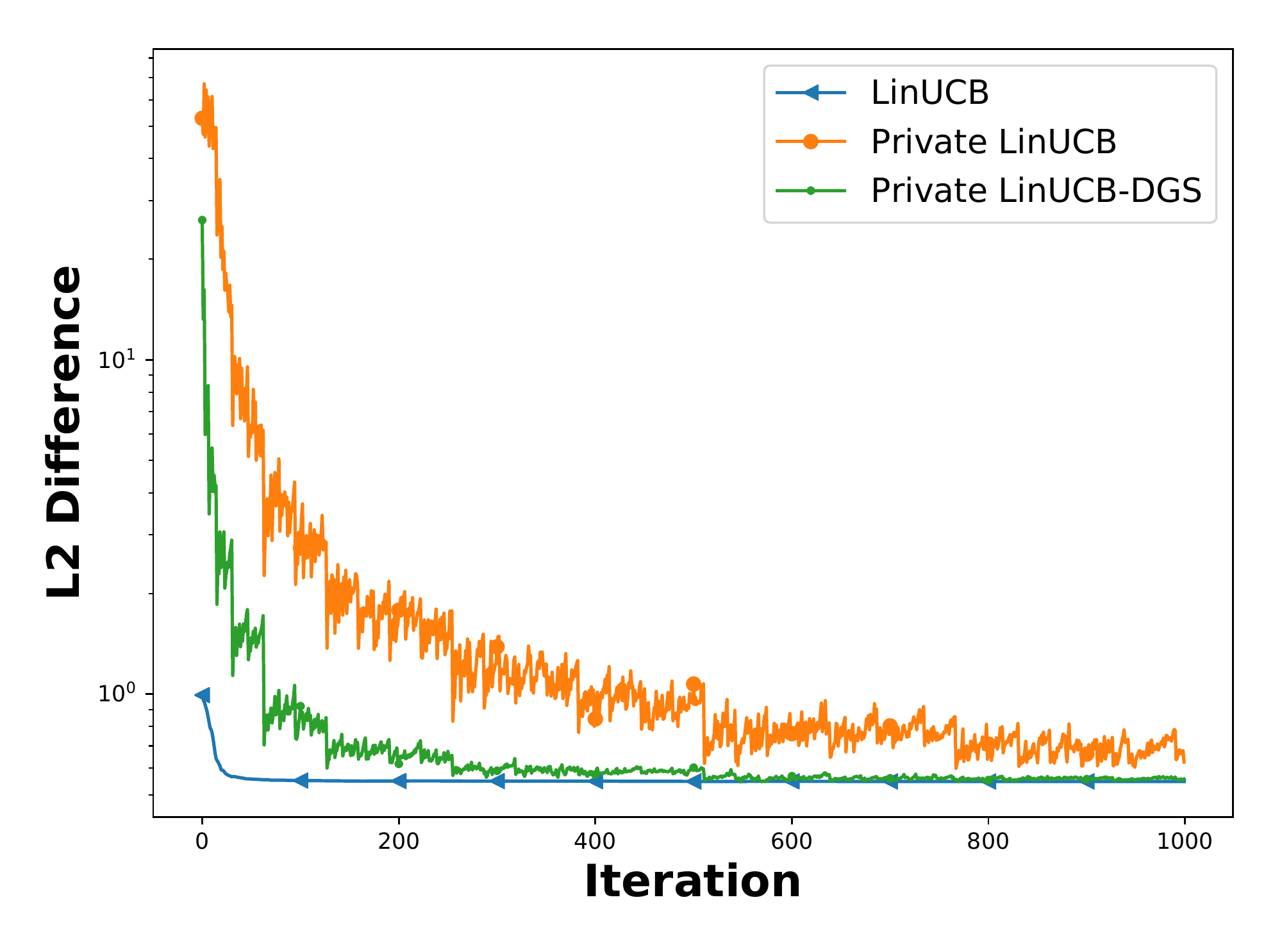} \\
(a) Cumulative regret on simulation data. & (b) Cumulative reward on LastFM.  &  (c) Parameter estimation error on simulation data.
\end{tabular}
\caption{Regret and parameter estimation quality on simulation and real-world dataset.} \label{fig:regret}
\vspace{-2mm}
\end{figure*}

\section{Experiment}\label{sec:exp}

We performed empirical evaluations of our proposed private LinUCB with dynamic global sensitivity (denoted as Private LinUCB-DGS) against two baseline algorithms:
\begin{itemize}
\item LinUCB \citep{LinUCB}: it selects an arm based on its upper confidence bound of the estimated reward with given context vectors. As a non-private algorithm, LinUCB does not inject any noise to its parameter estimation.
\item Private LinUCB \citep{neel2018mitigating}: it adds noise to $\bb_t$ using the tree-based mechanism with a constant sensitivity parameter, and incurs additional regret at the order of $\tilde O(\log^{2.5}{T}\sqrt{T}/\epsilon)$. 
\end{itemize}

\noindent\textbf{$\bullet$ Experiment setup.}
In our simulation-based study, we generate a size-$K$ ($K=1000$) arm pool $\cA$, in which each arm $a$ is associated with a $d$-dimensional ($d=10$) feature vector $\bx_a \in \mathbb{R}^d$ with $\lVert \bx_{a}\rVert_2 \leq 1$. Similarly, we create a set of ground-truth bandit parameters $\btheta^* \in \mathbb{R}^d$ with $\lVert \btheta^* \rVert_2 \leq 1$, which are not disclosed to the learners. Each dimension of both $\bx_a$ and $\btheta^*$ is drawn from a uniform distribution $U(0,1)$. At each round $t$, a randomly-sampled decision set with 10 arms from $\cA$ are disclosed to the learners for selection. The ground-truth reward $r_{a}$ is corrupted by Gaussian noise $\gamma \sim N(0,\sigma^2)$ before giving back to the learners and the standard deviation $\sigma$ is set to 0.5. 

We also experimented on a real-world dataset extracted from the music streaming service website Last.fm \footnote{\url{http://www.last.fm}}. The LastFM datasets is published by the HetRec 2011 workshop \footnote{http://grouplens.org/datasets/hetrec-2011}. This dataset contains 1,892 users and 17,632 items (artists). We used the information of ``listened artists'' of each user to create payoffs of recommendation candidates: if a user listened to an artist at least once, the payoff is 1, otherwise 0. Following the settings in \citep{Gang, wu2016contextual,wang2016learning}, we extracted the context features of artists and pre-processed the dataset in order to fit them into a contextual bandit setting. Specifically, we create a TF-IDF feature vector using associated tags (6,036 tags in total), which uniquely represents the context of that artist. PCA is used to reduce the dimensionality of the feature vectors to $d =25$. For a particular user $i$, we generate the candidate arm pool with size $|\cA_t|=25$ by first sampling one item from those non-zero reward candidate in user $i$'s past history, and then randomly fulfill the other $24$ from those zero-reward candidates from this user.

Notice that both synthetic data and real-world data met Assumption~\ref{assumption:context}. In synthetic data, the context features are sampled from a uniform distribution. For the LastFM dataset, the context vectors are generated by retaining the first 25 principal components of the original 6,036-dimension TD-IDF representation of items, which means the minimum eigenvalue of $\mathbb{E}[\mathbf{x}\mathbf{x}^\mt]$ is lower bounded.

\begin{figure*}[t]
\centering
\vspace{-2mm}
\begin{tabular}{ >{\centering\arraybackslash}m{4.6cm} >{\centering\arraybackslash}m{4.6cm}
>{\centering\arraybackslash}m{4.6cm}}
\hspace*{-0.4cm}
\includegraphics[width=5.6cm]{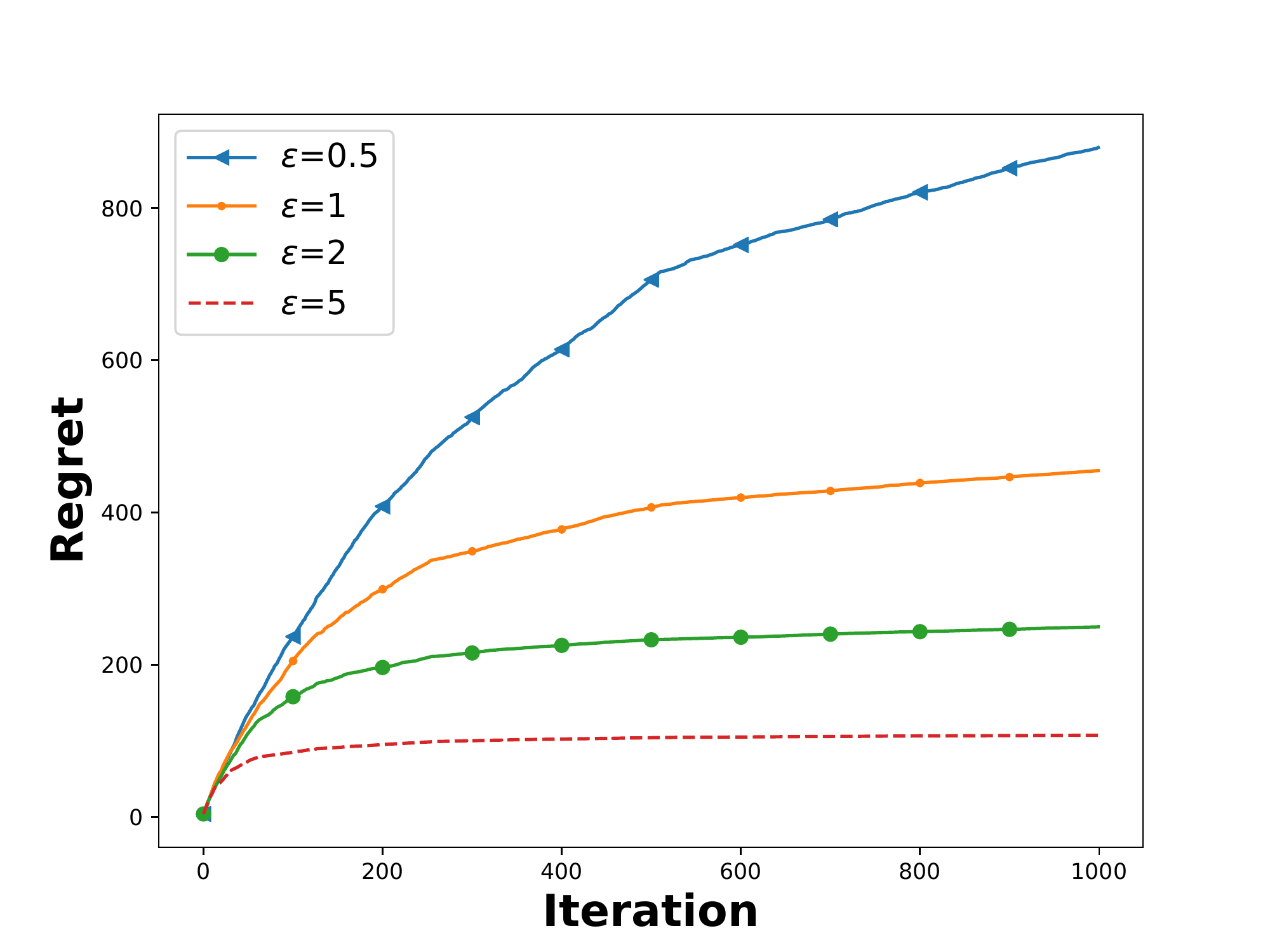}& 
\includegraphics[width=5.6cm]{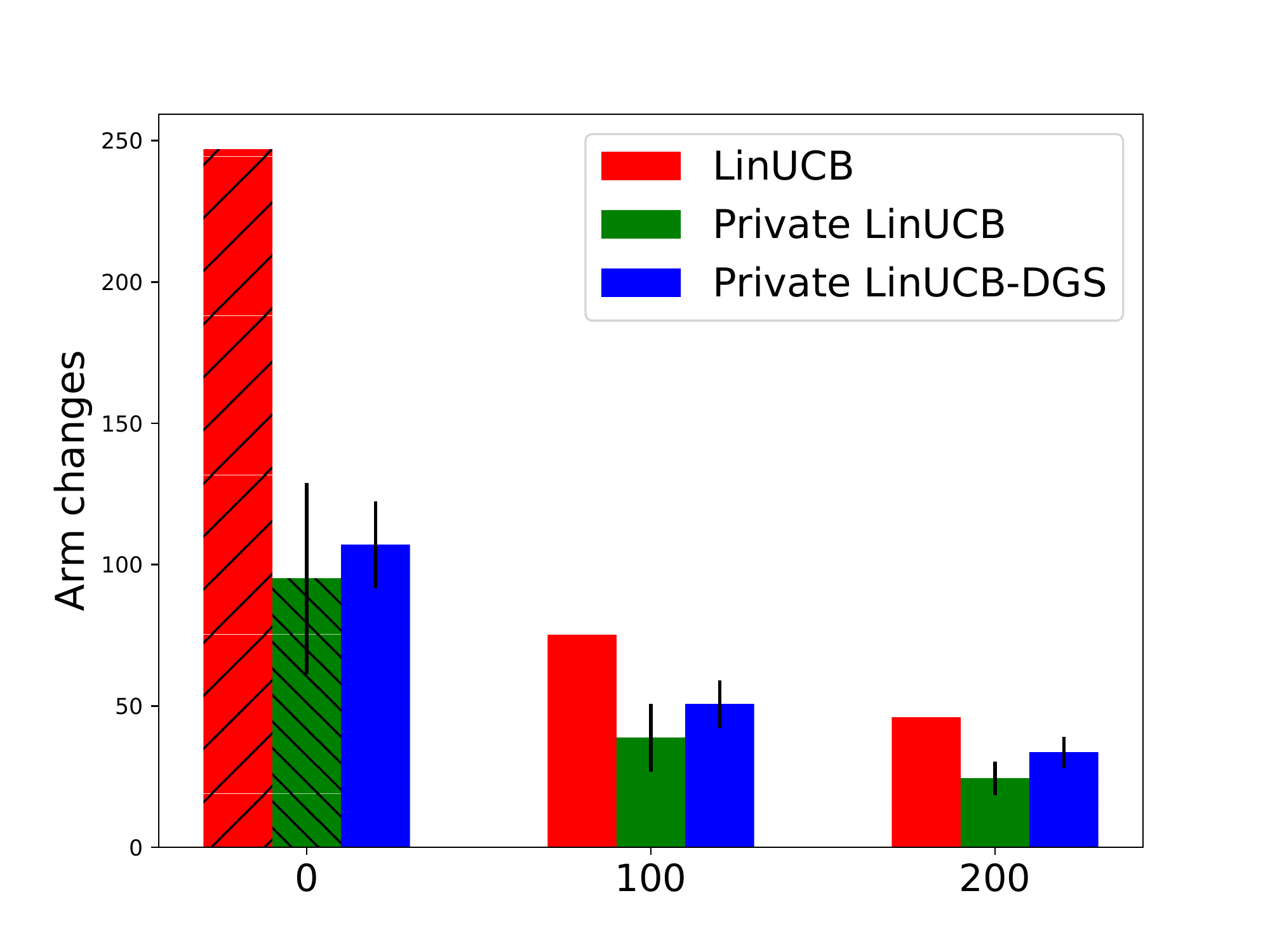}&  
\includegraphics[width=5.6cm]{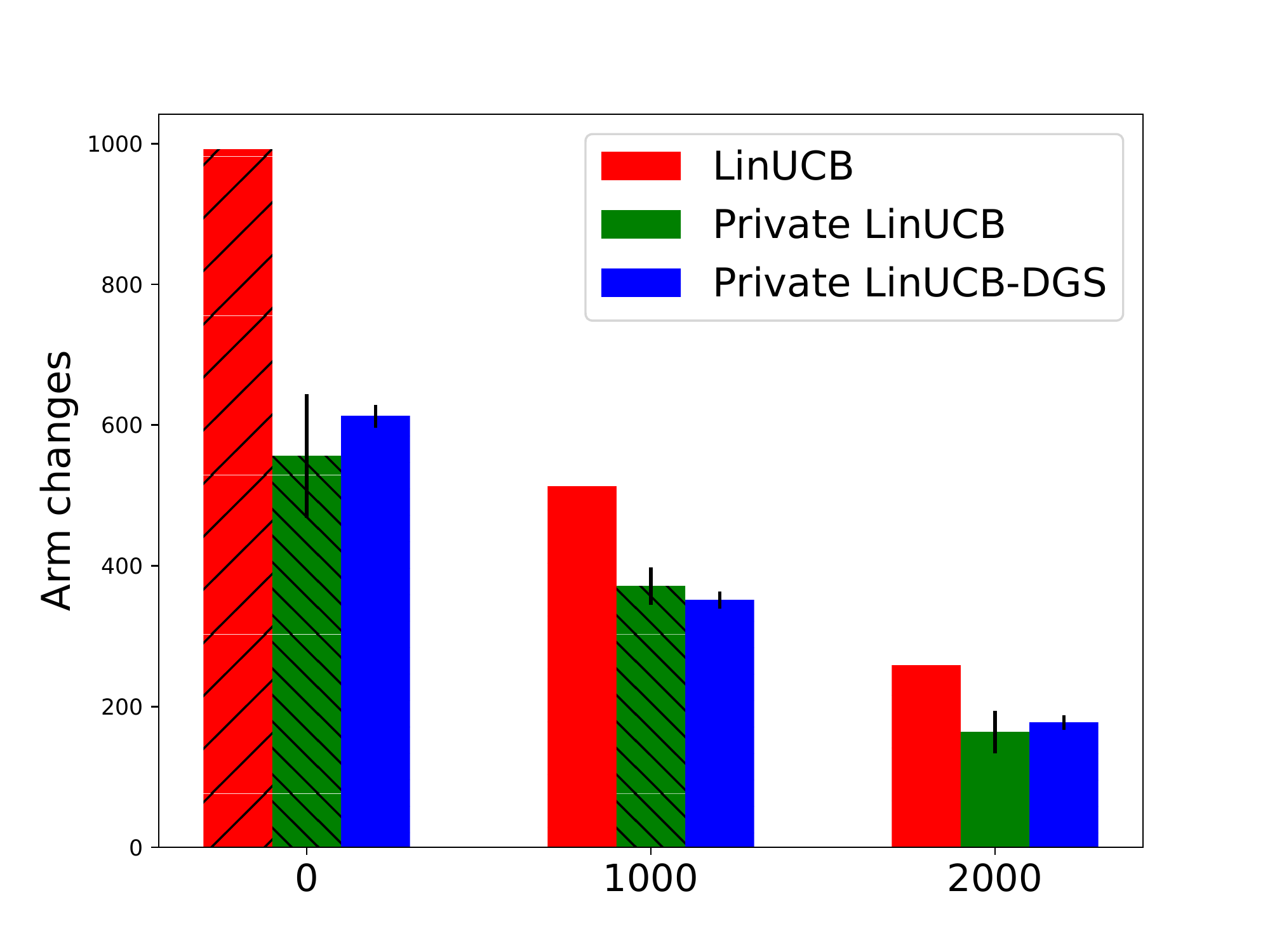}
 \\
(a) Effect of privacy level $\epsilon$. & (b) Change of selected arms on simulation with single reward difference. & (c) Change of selected arms on LastFM with single reward difference.
\end{tabular}
\caption{Detailed empirical analysis of privacy-preserving mechanism for LinUCB.} \label{fig:detail}
\vspace{-3mm}
\end{figure*}

\noindent\textbf{$\bullet$ Regret comparison.}
We first show the regret on simulation and real-world dataset in Figure \ref{fig:regret} (a) and (b). As in the real-world dataset we do not have an oracle policy, we instead report each learning algorithm's cumulative reward for comparison. We set the privacy level $\epsilon=2$ and  $\delta=0.1$ 
in the following experiments as default. On the LastFM dataset, we report the reward ratio normalized by the reward collected from a random policy; and the resulting performance curve is thus the higher the better. We observe that the non-private LinUCB performs better than its two private variants, since no noise is injected to LinUCB's model estimation. Our private LinUCB-DGS performs much better than private LinUCB on both datasets, as less noise is added .
These results support our theoretical analysis that with dynamic global sensitivity, our algorithm adds additional regret 
in the order of $\tilde O(\log{T}\sqrt{T}/\epsilon)$. 
Note that because the privacy notion are different in private LinUCB and our private LinUCB-DGS, here we focus on comparing the utility of these algorithms with the same $\epsilon$.

\noindent\textbf{$\bullet$ Parameter estimation quality.}
Figure \ref{fig:regret} (c) shows the parameter estimation quality of the three bandit algorithms, i.e., the L2 difference between the estimated bandit parameter $\hat\btheta^p_t$ and the ground-truth parameter $\btheta^*$. Compared with private LinUCB, our model parameter $\hat\btheta^p$ converges faster, which explains its improved performance in regret. We also notice that the convergence of private LinUCB oscillates more seriously. This is because private LinUCB injects noise with a larger variance (as we always introduce zero-mean Laplace noise) to model estimation. As a result, its quality of model parameters estimation varies significantly during online update, which directly leads to its worse regret. 

\noindent\textbf{$\bullet$ Effect of privacy level $\epsilon$.} In Figure \ref{fig:detail} (a), we show the regret of our algorithm with different privacy parameter $\epsilon$. We vary $\epsilon$ from $0.5$ to $5$. From the results, we notice a clear trade-off between the required privacy level $\epsilon$ and the resulting regret. Stronger privacy requirement (i.e., a smaller $\epsilon$) requires the privacy mechanism to inject more noise, which results in larger regret. This result also supports our theoretical analysis that in our solution is in the order of $O(\frac{1}{\epsilon})$.

\noindent\textbf{$\bullet$ Sensitivity in arm selection.} We now evaluate the algorithm's sensitivity to reward changes, which is exactly what differential privacy intends to protect in a bandit algorithm. Specifically, we fix the sequence of context and arms and only change reward feedback on one particular arm. Note that the feedback noise in simulation is also fixed for all arms ahead of time in this experiment. We compare how many arms are selected differently after the reward change by a bandit algorithm. In our simulation based study, we run each algorithm 500 iterations and change the reward at iteration $\{0, 100, 200\}$ by increasing the original reward by 0.5. For experiment on LastFM, we run 5,000 iterations, and flip reward on the selected item at iteration $\{0, 1000, 2000\}$. We run private LinUCB-DGS and private LinUCB-DGS for 5 times and report the mean and standard deviation. We run LinUCB just one time as it is a deterministic algorithm.

From Figure \ref{fig:detail} (b) and (c) we can observe that the number of changed arms in our solution is clearly smaller than its non-private counterpart, which suggests that private LinUCB-DGS is less sensitive to reward change comparing to original LinUCB. Private LinUCB has a similar number of changed arms than ours in average but with a larger variance. 
This suggests less stable utility provided by the private LinUCB solution. We also notice that when the reward change happens later, the change in all algorithms' arm selection gets smaller, which supports our motivation of dynamic sensitivity analysis that the algorithm's output is less sensitive to its input in the later stage.

\section{Conclusions \& Future Work}

In this paper, we first studied the dynamic nature of sensitivity of linear bandit models and presented a dynamic global sensitivity analysis of such an algorithm under the tree based mechanism, which leads to reduced noise addition for differential privacy. We then developed and analyzed a differentially private linear bandit algorithm based on the concept of dynamic global sensitivity. Our private linear bandit algorithm injects additional regret caused by privacy-preserving mechanism 
in $\tilde O(\log{T}\sqrt{T}/\epsilon)$ while guarantees $(\epsilon,\delta)$-differential privacy. Experimental results on both synthetic and real-world datasets demonstrated the advantage of our algorithm and supported our theoretical analysis.

One important property of bandit algorithms is their intrinsic noise/randomness introduced by exploration, such as the stochastic sampling of arms in Thompson Sampling and EXP3. Less explicit noise should be needed for differential privacy, because of such intrinsic randomness of an algorithm's output, e.g., free privacy. We plan to take advantage of it for better balancing privacy and utility in bandit algorithms. In addition, as we mentioned before our proposed dynamic global sensitivity analysis is not only limited to linear bandit algorithms. It is meaningful to explore its application in a broader context, such as logistic bandits and online convex optimization. We also notice that the lower bound of the differentially private linear bandit problem that protects privacy of reward is currently still an open problem, and it is important to investigate this lower bound to show the optimality of a private linear bandit algorithm.  

\begin{acks}
We thank the anonymous reviewers for their insightful comments. This work was supported by National Science Foundation Award IIS-2128019, IIS-2007492 and IIS-1553568, Google Faculty Research Award and Bloomberg Data Science Ph.D. Fellowship.

\end{acks}


\appendix
\section*{Appendix}
\subsection*{Missing Proofs}
In this section, we provide the detailed proofs of the lemma and theorems discussed in our paper. 

\begin{proof}[Proof of Lemma \ref{lemma:sensitivity}]
For any neighbouring reward sequences $\{S, S'\}$ that only differ at the data point $j$, i.e., $r_i = r'_{i}$ when $i\neq j$, denote $\bA_t = \lambda \bI+ \sum_{i=1}^{t} \bx_{a_i} \bx_{a_i}^\mt$, we have, 
\begin{align*}
&\lVert \hat\btheta_{t+1} - \hat\btheta'_{t+1}\rVert_2 \\=& \lVert (\lambda \bI+\sum_{i=1}^{t} \bx_{a_i} \bx_{a_i}^\mt)^{-1} \sum_{i=1}^{t} \bx_{a_i} r_i - (\lambda \bI+\sum_{i=1}^{t} \bx_{a_i} \bx_{a_i}^\mt)^{-1} \sum_{i=1}^{t} \bx_{a_i} r'_i  \rVert_2 \\
=& \lVert \bA_t^{-1} \bx_{a_i} (r_j - r'_j)  \rVert_2 \\
\leq& 2L\lambda_{max}(\bA_t^{-1})\\
\leq& \frac{2L}{\lambda_{min}(\bA_t)}
\end{align*}
$\lambda_{min}(\cdot)$ denotes the minimal eigenvalue of the input matrix. The third step is because of the assumption $\lVert\bx\rVert_2 \leq L$ and $\lVert r \rVert \leq 1$. To analyze the eigenvalue of $\bA_t$, we adopt the assumption from Theorem 1 of Gentile et al. \cite{gentile2014online} that context vectors $\{\bx_{1,t},..\bx_{K,t}\}$ are  i.i.d. conditioned on the algorithm's past actions and observed context. Based on Theorem 1 of \cite{gentile2014online}, $\lambda' = \lambda_0 t/4 - 8 \log((t+3)/\delta) -2\sqrt{t \log((t+3)/\delta)}$ and $\lambda' = \lambda_0 t/16$ when $t>32\log(1/\delta)/\lambda_0$. Thus we have $\lambda_{min}(\bA_t) \geq \lambda_0 t / 16$ when $t>32\log(1/\delta)/\lambda_0$. Substitute it back to the inequality and we finish the proof of Lemma \ref{lemma:sensitivity}.
\end{proof}

\begin{proof}[Proof of Theorem \ref{theorem:privacy}]
We can rewrite the estimated model parameter $\hat\btheta_t$ as sum statistics using the Sherman–Morrison formula by
\begin{align*}
\hat\btheta_{t+1} &= (\bA_t+\bx_{a_t}\bx_{a_t}^\mt)^{-1}(\bb_t+\bx_{a_t}r_{a_t}) \\
&=\bA_t^{-1}\bb_t + \Delta_{\bA_t} \bb_t + \bA_t^{-1} \bx_{a_t}r_{a_t} + \Delta_{\bA_t} \bx_{a_t}r_{a_t}\\
&=\hat\btheta_{t}+\Delta_{\bA_t} \bb_t + \bA_t^{-1} \bx_{a_t}r_{a_t} + \Delta_{\bA_t} \bx_{a_t}r_{a_t}\\
&:= \hat\btheta_{t} + \Delta(\hat\btheta_{t})
\end{align*}
where $\Delta_{\bA_t} = \frac{\bA_t^{-1}\bx_{a_t}\bx_{a_t}^\mt \bA_t^{-1}}{1 + \bx_{a_t}^\mt \bA_t^{-1}\bx_{a_t}}$. Note that we do not actually need to store nor calculate the partial sum in the tree; but we only use the tree mechanism to sample noise for each partial sum, which makes our implementation efficient. We notice that not saving partial sums on the tree makes our algorithm fail to satisfy the definition of \emph{pan privacy} \citep{dwork2010differential}. However if needed, one can always explicitly maintain a tree for $\Delta(\hat\btheta_{t})$ to achieve pan privacy, with a price of added computational cost.

For partial sum $\sum_{t=i}^j\Delta(\hat\btheta_{t})$, i.e., $\hat\btheta_{j}-\hat\btheta_{i}$, its sensitivity is bounded by $\hat\btheta_{j} - \hat\btheta'_{j}$ and further bounded by $\Delta_j$ as discussed in Lemma \ref{lemma:sensitivity}.
According to the definition of differential privacy, the partial sum is $(\epsilon', \delta')$ -differentially private if we add $\text{Lap}(\frac{\Delta_{j}}{\epsilon'})$ noise with our sensitivity bound that holds with probability at least $1-\delta'$.

We note that the incremental part $\Delta(\hat\btheta_{t})$ depends on not only  incoming $(\bx_{a_t}, r_{a_t})$, but also the historical statistics summarized in $\bA_t$. A similar problem also occurs in differentially private online convex optimization problem, where the current gradient depends on the choice of past gradients. We refer to Theorem 3 of \cite{thakurta2013nearly}, which uses the advanced composition theorem, to prove that the composition of $(\epsilon'=\epsilon/(\log T), \delta'=\delta/(\log T),)$-differentially private partial sums achieves $(\epsilon, \delta)$-differential privacy. 
\end{proof}
We clarify the analogy of problem structure between privacy of Online Convex Optimization (OCO) discussed in \cite{thakurta2013nearly} and privacy of contextual bandits as follows: our $\Delta\theta$ could be viewed similarly as gradient in OCO, which is protected by private sum statistics releasing. In OCO, gradient depends on current loss and historical evaluated parameters $\{w_t\}^T_{t=1}$, while in contextual bandits  $\Delta\theta$ depends on current reward and historical pulled arms $\{x_t\}^T_{t=1}$. An arm pulled in contextual bandits is analogous to parameter evaluated in OCO, and thus similar arguments and analysis apply.

\begin{proof}[Proof of Theorem \ref{theorem:utility}]
To bound the total amount of noise $\eta_t$ added by our tree mechanism, we first state the property of sum of independent Lapalace distributions (also stated in Lemma 2.8 of \cite{chan2011private}):
\end{proof}

\begin{lemma}
\label{lemma:laplace} 
With probability $1-\zeta$, sum of independent Lapalace distributions $\text{Lap}(a_i)$ is bounded by $O(\sqrt{\sum_i a_i^2}\log{\frac{1}{\zeta}})$
\end{lemma}

Let $b$ be the $\ceil{\log_2 T}+1$-bit binary representation of $t$. The noise introduced by tree-based mechanism with dynamic global sensitivity can be bounded by
\vspace*{-6pt}
\begin{align*}
\eta_t &= \sum_{b_i = 1} \text{Lap}(\frac{\Delta_{t-2^i+1}}{\epsilon'})\\
&= \sum_{b_i = 1} \text{Lap}(\frac{\Delta_{t-2^i+1}\log T}{\epsilon})\\
&\leq \log{\frac{1}{\zeta}} \sqrt{\frac{\log^2 T}{\epsilon^2}\left(\Delta^2_{\frac{1}{2}t'+1} + \Delta^2_{\frac{3}{4}t'+1} +\Delta^2_{\frac{7}{8}t'+1} + ... + \Delta^2_{t'}) \right)}\\
&\leq\frac{32L}{\lambda_0}\log{\frac{1}{\zeta}} \frac{\log T}{t\epsilon}\sqrt{\left((\frac{2}{1})^2 + (\frac{4}{3})^2 + ... \right)}\\
&=\frac{32L}{\lambda_0}\log{\frac{1}{\zeta}} \frac{\log T}{t\epsilon}\sqrt{\left((1 +\frac{1}{1})^2 + (1 + \frac{1}{3})^2 + ... \right)}\\
&\leq\frac{32L}{\lambda_0}\log{\frac{1}{\zeta}} \frac{\log T}{t\epsilon}\sqrt{2\log t + 2}
\vspace*{-6pt}
\end{align*} 
where $t'>t$ is the smallest number that is the power of 2 after $t$. The last inequality holds because there are at most $\log t$ terms in the summation, and $(\frac{1}{1})^2+(\frac{1}{3})^2+..$ is upper bounded by 2.
As a result, we conclude that the added noise $\eta_t$ is in the order of $O(\frac{L}{\epsilon}\log T \sqrt{\log t}\log{\frac{1}{\zeta}}/t )$ with probability $1-\zeta$.

\begin{proof}[Proof of Theorem \ref{theorem:regret}]
We first bound the one-step regret at time $t$ as,
\vspace*{-6pt}
\begin{align*}
\textit{regret}(t) &= \bx^{\mt}_{a^*_t}\btheta^* - \bx^{\mt}_{a_t}\btheta^* ~~// \text{Definition of regret}\\
&\leq \bx^{\mt}_{a^*_t}\hat\btheta^p_t + \text{CB}_t(\bx_{a^*_t}) - \bx^{\mt}_{a_t}\btheta^*  ~~// \text{Definition of confidence bound} \\
&\leq \bx^{\mt}_{a_t}\hat\btheta^p_t + \text{CB}_t(\bx_{a_t}) - \bx^{\mt}_{a_t}\btheta^* ~~// \text{Arm selection strategy} \\
&\leq 2\text{CB}_t(\bx_{a_t})  
\end{align*}
The pseudo-regret up to time $T$ is thus bounded by,
\begin{align*}
\text{Regret}(T) &= \sum_{t=1}^T regret(t) \\
&\leq \sum_{t=1}^T 2\text{CB}_t(\bx_{a_t}) \\
&= 2\sum_{t=1}^T \alpha_t \lVert \bx_{a_t} \rVert_{\bA^{-1}_t} \\
&\leq 2\sum_{t=1}^T \left( \lVert\hat\btheta_t^p - \hat\btheta_t\rVert_{\bA_t} + \lVert\hat\btheta_t - \btheta^*\rVert_{\bA_t} \right)  \lVert \bx_{a_t} \rVert_{\bA^{-1}_t} \\
 &\leq 2\sqrt{\sum_{t=1}^T \lVert\hat\btheta_t - \btheta^*\rVert^2_{\bA_t} \sum_{t=1}^T  \lVert \bx_{a_t} \rVert^2_{\bA^{-1}_t}}\\ 
&+ 2\sqrt{\sum_{t=1}^T\lVert \eta_t \rVert^2_{\bA_t} \sum_{t=1}^T  \lVert \bx_{a_t} \rVert^2_{\bA^{-1}_t}} \\
 &\leq 2\sqrt{T \lVert\hat\btheta_T - \btheta^*\rVert^2_{\bA_T} \sum_{t=1}^T  \lVert \bx_{a_t} \rVert^2_{\bA^{-1}_t}}\\ 
&+ 2\sqrt{T \max_{t'}\lVert \eta_{t'} \rVert^2_{\bA_{t'}} \sum_{t=1}^T  \lVert \bx_{a_t} \rVert^2_{\bA^{-1}_t}}
\end{align*}

In the fourth step, we separate the upper bound of $\alpha_t$ into two terms, among which one is the confidence bound of the original LinUCB and the other is the bound of injected noise. The fifth step follows the Cauchy–Schwarz inequality. Using Lemma 11 of \cite{Improved_Algorithm}, we have $\sum_t \lVert \bx_{a_t} \rVert^2_{\bA^{-1}_t} \leq \sqrt{d\log{(\lambda+\frac{LT}{d})}} $.
The bound of $ \lVert\hat\btheta_T - \btheta^*\rVert^2_{\bA_T} \leq \sqrt{d\log{\frac{1+TL^2\lambda}{\zeta}}}+\sqrt{\lambda}S $ follows self-normalized bound for martingales of Lemma 9 in \cite{Improved_Algorithm}. We bound $\max_{t'}\lVert \eta_{t'} \rVert^2_{\bA_{t'}}$ when $\lambda \geq \lambda_0t/16$, i.e.,  $t > 32\log(1/\delta)/\lambda_0$, by taking the derivative of $\frac{\log^{1.5}t}{\sqrt{t}}$ and find its maximum value at $t=e^3$, since it is a concave function. Substitute it back and we get its bound of $\frac{3^{1.5}}{e^{1.5}}$. The regret for  $t\leq 32\log(1/\delta)/\lambda_0$ is at most $32L\log(1/\delta)/\lambda_0$.  Combining all these terms together, we prove the regret bound of our developed private LinUCB based on global dynamic sensitivity analysis. 
\end{proof}

\begin{proof}[Proof of Corollary \ref{corollary:eigen}]
When every context vector $\bx$ is perturbed by Gaussian noise $z\sim \mathcal{N}(0, \sigma^2) $, $\mathbb{E}[(\bx+z)(\bx+z)^\mt]$ has minimal eigenvalue $\lambda_0$ where $\lambda_0 \geq \Omega(\sigma^2)$. Thus Lemma 1 and 2 hold, and the privacy guarantee directly follows Theorem~\ref{theorem:privacy}.  Following Lemma 3.3 in \cite{kannan2018smoothed}, we have $ \lVert\hat\btheta_T - \btheta^*\rVert^2_{\bA_T} \leq \frac{\sqrt{d\log{\frac{1+TL^2\lambda}{\zeta}}}+\sqrt{\lambda}S}{\sigma^2} $ . Plug the result into the analysis of Theorem \ref{theorem:regret} and we have the regret bound.
\end{proof}

\balance
\bibliographystyle{ACM-Reference-Format}
\bibliography{main.bib}

\end{document}